\documentclass[11pt]{article}
\usepackage{fullpage}
\usepackage{xfrac}

\usepackage{tcolorbox}

\title{Boosting with List-Decodable Codes}
\author{}

\usepackage{liyang}

\allowdisplaybreaks

\usepackage{algorithm}
\usepackage{caption}
\usepackage[noend]{algpseudocode}

\usepackage{tikz}
\usepackage{pgfplots}
\usepackage{varwidth}

\usetikzlibrary{positioning, shapes.geometric, arrows.meta, calc}
\usetikzlibrary{decorations.pathreplacing}

\newcommand{\cA}{\mathcal{A}}
\newcommand{\cB}{\mathcal{B}}
\newcommand{\cC}{\mathcal{C}}
\newcommand{\cD}{\mathcal{D}}

\newcommand{\cF}{\mathcal{F}}
\newcommand{\cG}{\mathcal{G}}

\newcommand{\cN}{\mathcal{N}}

\newcommand{\cQ}{\mathcal{Q}}

\newcommand{\cU}{\mathcal{U}}

\newcommand{\cW}{\mathcal{W}}
\newcommand{\cX}{\mathcal{X}}

\newcommand{\pmo}{\{\pm 1\}}
\newcommand{\pmon}{\pmo^n}

\newcommand{\Enc}{\textup{\textsc{Enc}}}
\newcommand{\Dec}{\textup{\textsc{Dec}}}

\newcommand{\Xor}{\textup{\textsc{Xor}}}

\newcommand{\Cons}{\text{\textsc{Cons}}}
\newcommand{\ConsFrac}{\overline{\text{\textsc{Cons}}}}
\newcommand{\Err}{\text{\textsc{Err}}}
\newcommand{\ErrFrac}{\overline{\text{\textsc{Err}}}}
\newcommand{\ileave}{\textup{\textsc{Il}}}

\newcommand{\DKL}{D_{\textup{KL}}}

\def\colorful{1}
\def\showchangedthings{0}

\ifnum\colorful=1

\fi
\ifnum\colorful=0

\fi

\ifnum\showchangedthings=1
\newcommand{\changedthing}[1]{{\color{red}{#1}}}
\fi
\ifnum\showchangedthings=0
\newcommand{\changedthing}[1]{{#1}}
\fi

\newtheorem*{theorem*}{Theorem}

\author{ 
Addison Prairie \vspace{6pt} \\ 
\hspace{-10pt} {{\sl Stanford}} \and 
Li-Yang Tan \vspace{6pt} \\
\hspace{-15pt} {{\sl Stanford}}
}

\crefname{claim}{claim}{claims}
\Crefname{claim}{Claim}{Claims}

\begin{document}

\maketitle
\begin{abstract}%
    Boosting is a fundamental technique for generically improving the accuracy of learning algorithms (Schapire 1989). Existing boosting algorithms construct a strong learner  using $O(\log(\frac{1}{\epsilon})/\gamma^2)$ calls to a $\gamma$-advantage weak learner, and this round complexity is known to be optimal for generic boosters that succeed on all concept classes (Freund 1995).

    We show that this lower bound can be circumvented for concept classes that satisfy a mild closure property. Specifically, we present a new boosting algorithm that, for any class $\mathcal{F}$ closed under $O(\log \frac{1}{\gamma})$-XOR, strong learns $\mathcal{F}$ using $O(\log \frac{1}{\epsilon})$ calls to a $\gamma$-advantage weak learner and a single batch of $\tilde{O}(\log(\frac{1}{\epsilon})/\gamma^2)$ additional  samples.
    
    Our algorithm arises from a new and simple connection between boosting and list-decodable codes. Viewing the target function as a message, we run the weak learner on its encoding and view the resulting weak hypothesis as a corrupted codeword. Feeding this corrupted codeword to a list decoder, we obtain a small list of candidate hypotheses, at least one of which is a strong hypothesis for the original function. Using additional samples, we identify and output this strong hypothesis.
\end{abstract}

\thispagestyle{empty}
 \newpage 
 \setcounter{page}{1}

\section{Introduction}
Boosting is a celebrated method for systematically improving the accuracy of  learning algorithms. Given black-box access to a \textsl{weak learner}---an algorithm that produces hypotheses with accuracy slightly better than trivial---for a concept class, a boosting algorithm produces a \textsl{strong learner} that achieves arbitrarily high accuracy. Boosting algorithms typically work by sequentially running the weak learner multiple times and aggregating the weak hypotheses. Originally introduced by Schapire \cite{SchapireBoosting}, boosting is notable as a technique that has proven successful both in theory and in practice.

A key measure of a boosting algorithm's efficiency is the number of calls it makes to the weak learner, also known as its \textsl{round complexity}. To strong learn with accuracy $1 - \eps$ given a weak learner with accuracy $\frac{1}{2} + \gamma$, existing boosting algorithms (e.g., AdaBoost, \cite{FREUND1997119}) make $O(\log(\frac{1}{\eps})/\gamma^2)$ calls to the weak learner. Early work of Freund  showed that this is essentially optimal:

\begin{theorem*}[\cite{FREUND1995256}]
For any boosting algorithm $\cB$, there exists a concept class $\cF = \{\cF_n\}_{n \in \N}$ and a $(\frac{1}{2} + \gamma)$-accuracy weak learner $\cW$ such that $\cB$ must call $\cW$ at least $\Omega(\log(\frac{1}{\eps})/\gamma^2)$ times to learn $\cF$ to $1-\eps$ accuracy. 
\end{theorem*}

However, the hard concept class $\cF$ used in this lower bound is highly artificial: for each $n \in \N$, the slice $\cF_n$ consists of a single function $f : \pmon \rightarrow \pmo$ chosen uniformly at random. As a result, it does not rule out the possibility of more round-efficient boosting for the more structured concept classes that arise in theory and practice.

In this work, we evade Freund's lower bound by restricting the concept class being learned. Our approach applies to any class $\cF = \{\cF_n\}_{n \in \N}$ satisfying a mild closure assumption: we say that $\cF$ is \textsl{closed under} $k$-\Xor{} if, for every $n \in \N$ and $f \in \cF_n$, the function $f^{\oplus k}(x_1, \ldots, x_k) = f(x_1) \oplus \cdots \oplus f(x_k)$ is in $\cF_{n \cdot k}$. Under such an assumption on $\cF$, we give a boosting algorithm that makes significantly fewer calls to $\cW$: 

\medskip 
\begin{tcolorbox}[colback = white,arc=1mm, boxrule=0.25mm]
    \begin{theorem} [see \Cref{thm:main-theorem-formal} for formal statement]\label{thm:informal-main-thm}
    There exists an efficient boosting algorithm that, for any $\gamma, \eps > 0$, given a $\gamma$-weak learner $\cW$ for a concept class $\cF$ closed under $O(\log \frac{1}{\gamma})$-\Xor{}, learns $\cF$ to accuracy $1 - \eps$ with $O(\log \frac{1}{\eps})$ calls to $\cW$ and $\Tilde{O}(\log(\frac{1}{\eps})/\gamma^2)$ additional labeled samples.
\end{theorem}
\end{tcolorbox}
\medskip

We make the following remarks about \Cref{thm:informal-main-thm}.

\paragraph{Additional Labeled Samples.} Here, by {additional labeled samples}, we mean samples beyond those implicitly required for $O(\log \frac{1}{\eps})$ calls to the weak learner $\cW$. Our $\tilde{O}(\log(\frac{1}{\eps})/\gamma^2)$ additional samples \changedthing{are} modest compared to existing boosting algorithms that make $O(\log(\frac{1}{\eps})/\gamma^2)$ calls to $\cW$, each of which would seemingly require at least one additional sample. Thus, the total number of labeled samples used by our algorithm is comparable to or less than that of existing boosting algorithms.

\paragraph{The Tradeoff we Incur.} Our algorithm achieves a reduction in the number of calls to $\cW$ at the cost of each call requiring more resources. When learning an unknown $f \in \cF_n$, it calls $\cW$ on $f^{\oplus k} \in \cF_{n \cdot k}$ for $k := O(\log \frac{1}{\gamma})$. Thus, each call to $\cW$ requires more samples and time than a call made by an existing boosting algorithm: to strong learn a function over an instance space $\cX$, our booster calls the weak learner on a function on $\cX^k$; for example, when learning functions over $\pmon$ or $\R^n$, we call the weak learner on $\pmo^{n \cdot k}$ or $\R^{n \cdot k}$, respectively. However, if $\cW$ runs in polynomial time, this results in a rather modest $\polylog(\frac{1}{\gamma})$ increase in samples and time per call to $\cW$, while significantly reducing the number of calls to $\cW$.

\paragraph{Mildness of Our Closure Assumption.} In the setting of $\gamma = n^{-O(1)}$ (the standard setting for weak learning), we only require closure under $O(\log n)$-\Xor. Since the \Xor{} of $O(\log n)$ bits can be computed by a $\poly(n)$-sized DNF/CNF, our boosting algorithm applies to a range of concept classes, even computationally weaker ones such as small, constant-depth circuits.

\subsection{Uniform-Distribution Boosting}

One downside of most existing boosting algorithms is that, even to strong learn with respect to a single distribution (e.g., uniform), they require a weak learner that succeeds over arbitrary distributions. A strength of our approach is that it can be made distribution-specific: to achieve strong learning over the uniform distribution, it suffices to have a weak learner for the uniform distribution.

\medskip 
\begin{tcolorbox}[colback = white,arc=1mm, boxrule=0.25mm]
    \begin{theorem}[see \Cref{thm:main-dist-specific} for formal statement] \label{thm:informal-dist-specific-thm}
        There exists an efficient boosting algorithm that, for any $\gamma, \eps > 0$, given a uniform distribution $\gamma$-weak learner $\cW$ for a concept class $\cF$ closed under $O(\log (\frac{1}{\gamma})/\eps)$-\Xor{}, learns $\cF$ to accuracy $1 - \eps$ over the uniform distribution with only $1$ call to $\cW$ and \changedthing{$\tilde{O}(\log(\frac{1}{\gamma})/\eps)$} additional labeled samples.
    \end{theorem}
\end{tcolorbox}
\medskip

\Cref{thm:informal-main-thm,thm:informal-dist-specific-thm} are incomparable, as they apply to different settings of boosting. Quantitatively, \Cref{thm:informal-dist-specific-thm} exhibits a worse dependence on $\eps$ but a better dependence on $\gamma$ in the additional samples required. Moreover, it makes only a single call to $\cW$, whereas \Cref{thm:informal-main-thm} requires $O(\log \frac{1}{\eps})$ calls.

\subsection{Other Related Work} \label{sec:related-work}


\paragraph{Boosting Simple Learners.}
Recent work of Alon, Gonen, Hazan and Moran \cite{boosting-simple-learners} gives a boosting algorithm that learns to $99\%$ accuracy using only $O(\frac{1}{\gamma})$ calls to the weak learner $\cW$, under the assumption that $\cW$ produces hypotheses from a class of constant VC dimension. While this assumption and our closure assumption are incomparable, we remark that our boosting algorithms are efficient, running in time $\poly(n, \frac{1}{\gamma}, \frac{1}{\eps})$; in contrast, their algorithm runs in time exponential in $\frac{1}{\gamma}$.


\paragraph{Uniform-Distribution Boosting.}
Boneh and Lipton \cite{Boneh-Lipton} design a uniform-distribution boosting algorithm which similarly relies on the concept class being closed under \Xor. However, their result uses $\tilde{O}(1/(\gamma^4 \eps^2))$ additional samples and requires closure under $O(1/\gamma^2 \eps)$-\Xor{}. Comparing this to \Cref{thm:informal-dist-specific-thm}, our exponentially better dependence on $\gamma$ in the closure assumption means that our result can be applied to a far broader range of parameters and concept classes. For example, when $\gamma = n^{-O(1)}$ and we aim for $99\%$ accuracy, their algorithm requires closure under $\poly (n)$-\Xor, ruling out computationally weaker classes such as $\textup{AC}^0$. In contrast, our \Cref{thm:informal-dist-specific-thm} only requires closure under $O(\log n)$-\Xor, and therefore \textsl{does} apply to such classes.

\paragraph{Distribution-Specific Agnostic Boosting.} Another approach to distribution-specific boosting is to place an assumption on the weak learner $\cW$ rather than the concept class $\cF$. In particular, Kalai and Kanade \cite{kalai-kanade}, and later Feldman \cite{feldman}, developed algorithms assuming that $\cW$ is \textsl{agnostic}, a strong notion of noise tolerance. In this setting, the weak learner is not promised that the data it receives are labeled by a function in the target class $\cF$. Instead, it must perform well even when the labels have been corrupted by arbitrary (potentially adversarial) noise; see \cite{feldman} for more details. While our assumption on $\cF$ is incomparable to the assumption that $\cW$ is agnostic, we remark that \Cref{thm:informal-dist-specific-thm} achieves the same goal with only \textsl{one} call to $\cW$, in contrast to the $O(\gamma^{-2})$ calls required by the distribution-specific agnostic booster above.

\section{Technical Overview}

\paragraph{The standard approach to boosting.} To understand how our approach to boosting differs from existing techniques, we first explain why existing boosting algorithms require distribution-free weak learners and many calls to $\cW$. Such algorithms work by calling the weak learner on a sequence of distributions $\cD_1, \ldots, \cD_T$ to obtain hypotheses $h_1, \ldots, h_T$, then combining these hypotheses via a weighted majority. Starting with $\cD_1 := \cD$ (the target distribution), subsequent distributions are chosen so that $\cD_i$ places more weight on samples misclassified by $h_1, \ldots, h_{i - 1}$. By carefully adjusting each $\cD_i$ and weighting each hypothesis, AdaBoost iteratively reduces error with each call to $\cW$.

\begin{figure}[h]
    \centering
\begin{tikzpicture}[
    distnode/.style={
    },
    weaklearner/.style={
        rectangle,
        thick,
        draw,
        rounded corners=3pt,
        inner sep=7.5pt,
        align=center,
    },
    hypothesis/.style={
    },
    arrow/.style={
        -Stealth,
        thick,
        rounded corners=6pt
    },
    node distance=1cm and 1.5cm
]
\def\horizspace{1cm}
\def\lefthorizspace{1.25cm}
\def\righthorizspace{.75cm}
\def\verticalspace{-1.25cm}
\def\verticalspacedots{-2cm}
\def\filterverticaloffset{0.2cm}

    \node[distnode] (D1) {$S_1$};
    \node[weaklearner, anchor=west] (W1) at ($(D1.east) + (\lefthorizspace, 0)$) {$\text{weak learner }\mathcal{W}$};
    \node[hypothesis, anchor=west] (H1) at ($(W1.east) + (\righthorizspace, 0)$) {$h_1$};

    \node[distnode] (D2)  at ($(D1.south) + (0,\verticalspace)$) {$S_2$};
    \node[weaklearner, anchor=west] (W2) at ($(D2.east) + (\lefthorizspace, 0)$) {$\text{weak learner }\mathcal{W}$};
    \node[hypothesis, anchor=west] (H2) at ($(W2.east) + (\righthorizspace, 0)$) {$h_2$};

    \node[distnode] (D3) at ($(D2.south) + (0,\verticalspace)$) {$S_3$};
    \node[weaklearner, anchor=west] (W3) at ($(D3.east) + (\lefthorizspace, 0)$) {$\text{weak learner }\mathcal{W}$};
    \node[hypothesis, anchor=west] (H3) at ($(W3.east) + (\righthorizspace, 0)$) {$h_3$};

    \node[distnode] (DT) at ($(D3.south) + (0,\verticalspacedots)$) {$S_T$};
    \node[weaklearner, anchor=west] (WT) at ($(DT.east) + (\lefthorizspace, 0)$) {$\text{weak learner }\mathcal{W}$};
    \node[hypothesis, anchor=west] (HT) at ($(WT.east) + (\righthorizspace, 0)$) {$h_T$};

    \draw[arrow] (D1) -- (W1);
    \draw[arrow] (W1) -- (H1);

    \draw[arrow] (D2) -- (W2);
    \draw[arrow] (W2) -- (H2);

    \draw[arrow] (D3) -- (W3);
    \draw[arrow] (W3) -- (H3);

    \draw[arrow] (DT) -- (WT);
    \draw[arrow] (WT) -- (HT);

    \coordinate (turnpoint) at ($(H1.east) + (0.6, 0)$);

    \node[hypothesis] (Hout) at ($(turnpoint) + (0, -6.5)$) {$h_{\text{out}}$};

    \coordinate (stagger1) at ($(Hout.north) + (0, 4.5)$);
    \coordinate (stagger2) at ($(Hout.north) + (0, 3.0)$);
    \coordinate (stagger3) at ($(Hout.north) + (0, .5)$);

    \draw[arrow,-] (H1.east) -- (H1.east -| turnpoint) -- (stagger1);
    \draw[arrow,-] (H2.east) -- (H2.east -| turnpoint) -- (turnpoint |- stagger2) -- (stagger2);
    \draw[arrow,-] (H3.east) -- (H3.east -| turnpoint) -- (turnpoint |- stagger3) -- (stagger3);
    \draw[arrow] (HT.east) -- (HT.east -| turnpoint) -- (Hout.north);

    \coordinate (H1D2Height) at ($.5*(H1.south) + .5*(H2.north)$);
    \coordinate (H2D3Height) at ($.5*(H2.south) + .5*(H3.north)$);

    \coordinate (H3DTHeight) at ($.5*(H2.south) + .5*(H3.north) - (H2.south) + (H3.south)$);

    \coordinate (filterX) at ($.5*(D1.east) + .5*(W1.west) + (-.1, 0)$);

    \node[fill=white,draw=white] (R2) at ($(filterX |- D2.east)$) {$\backsim$};
    \node[fill=white,draw=white] (R3) at ($(filterX |- D3.east)$) {$\backsim$};
    \node[fill=white,draw=white] (RT) at ($(filterX |- DT.east)$) {$\backsim$};

    \draw[arrow] (H1.south) -- (H1D2Height -| H1.south) -- (H1D2Height -| W1.south) -- (filterX |- H1D2Height) -- ($(filterX |- D2.east) + (0,\filterverticaloffset)$);

    \draw[arrow] (H2.south) -- (H2D3Height -| H2.south) -- (H2D3Height -| W1.south) -- (filterX |- H2D3Height) -- ($(filterX |- D3.east) + (0,\filterverticaloffset)$);

    \draw[arrow] (H3.south) -- (H3DTHeight -| H3.south) --(H3DTHeight -| W1.south) -- (filterX |- H3DTHeight) -- ($(filterX |- DT.east) + (0,\filterverticaloffset)$);

    \coordinate (dotheight) at ($.5*(W3.south) + .5*(WT.north) + (0,-.2)$);

    \draw[fill=white, draw=white] ($(dotheight)+(-3.75,.2)$) rectangle ($(dotheight)+(4.25,-.4)$);

    \node (dots1) at (dotheight -| filterX) {$\vdots$};
    \node (dots2) at (dotheight -| Hout.north) {$\vdots$};

\end{tikzpicture}

    \caption{\parbox{0.8\textwidth}{\small \textbf{Figure 1.} Existing boosting algorithms repeatedly call the weak learner, using previous hypotheses to reweight ($\backsim$) the distribution of each dataset $S_i$ fed to $\cW$. The output $h_{\textup{out}}$ is then constructed by aggregating the weak hypotheses $h_1, \ldots, h_T$.}}

\end{figure}
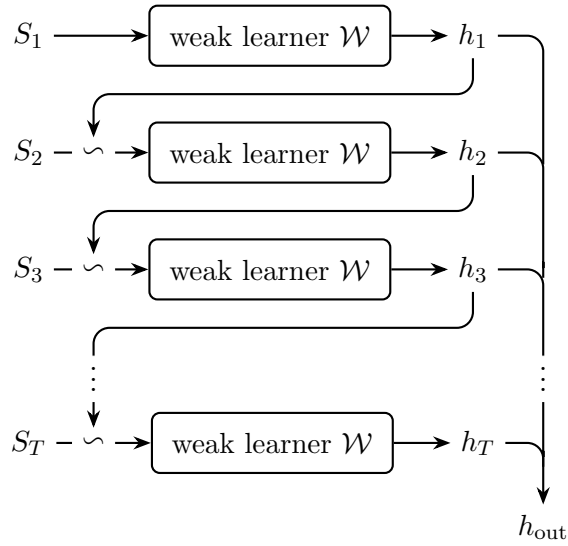

\newpage
\noindent
\paragraph{Our approach.} Rather than modifying the distributions given to $\cW$, we modify the function being learned. An initial dataset $S_1$ labeled by $f$ is used to construct a dataset labeled by a new function $f'$, which is then fed to $\cW$. This outputs a hypothesis $h'$ weakly computing $f'$, which we use to construct a strong hypothesis $h_{\textup{out}}$ for $f$ using additional data $S_2$. As described in the next section, the way we convert $f$ to $f'$ and recover $h_{\textup{out}}$ from $h'$ is grounded in a new connection between boosting and error correcting codes.

\begin{figure}[h]
    \centering
\begin{tikzpicture}[
    distnode/.style={
    },
    weaklearner/.style={
        rectangle,
        thick,
        draw,
        rounded corners=3pt,
        inner sep=7.5pt,
        align=center,
    },
    hypothesis/.style={
    },
    arrow/.style={
        -Stealth,
        thick,
        rounded corners=6pt
    },
    node distance=1cm and 1.5cm
]
\def\horizspace{1.5cm}
\def\lefthorizspace{1.5cm}
\def\righthorizspace{.75cm}
\def\verticalspace{-1.5cm}
\def\verticalspacedots{-2.5cm}

    \node[distnode] (D1) {$S_1$};

    \node[weaklearner, anchor=west] (Enc) at ($(D1.east) + (.75, 0)$) {$\textsc{Enc}$};

    \node[weaklearner, anchor=west] (W) at ($(Enc.east) + (.75, 0)$) {$\text{weak learner }\mathcal{W}$};

    \node[hypothesis, anchor=west] (Hp) at ($(W.east) + (.75, 0)$) {$h'$};

    \node[weaklearner, anchor=west] (Dec) at ($(Hp.east) + (.75, 0)$) {$\textsc{Dec}$};


    \coordinate (splitpnt) at ($(Dec.south) + (0, -.25)$);

    \coordinate (splitl) at ($(splitpnt) + (-.5, 0)$);
    \coordinate (splitr) at ($(splitpnt) + (.5, 0)$);

    \node[hypothesis] (H1) at ($(splitl) + (-.3, -.75)$) {$h_{1}$};
    \node[hypothesis] (HL) at ($(splitr) + (0.3,-.75)$) {$h_{L}$};

    \node[hypothesis] (Hdots) at ($.5*(H1.east)+.5*(HL.west)$) {$\cdots$};

    \coordinate (listdstart) at ($.5*(H1.south) + .5*(HL.south) + (0, -.2)$);

    \node[hypothesis, anchor=north] (H) at ($(listdstart) + (0,-1.25)$) {$h_{\textup{out}}$};

    \coordinate (filterpt) at ($.5*(listdstart)+ .5*(H.north) + (0, .1)$);

    \node[hypothesis, anchor=west] (D2) at ($(filterpt) + (1.5, 0)$) {$S_2$};

    \draw[arrow] (D1.east) -- (Enc.west);
    \draw[arrow] (Enc.east) -- (W.west);
    \draw[arrow] (W.east) -- (Hp.west);
    \draw[arrow] (Hp.east) -- (Dec.west);
    \draw[arrow,-] (Dec.south) -- (splitpnt);
    \draw[arrow,-] (splitl) -- (splitr);

    \draw[arrow] (splitl) -- (splitl -| H1.north) -- (H1.north);
    \draw[arrow] (splitr) -- (splitr -| HL.north) -- (HL.north);

    \draw[arrow] (splitpnt) -- ($(Hdots.north) + (0, .1)$);

    \draw[decorate,decoration={brace,amplitude=6pt,mirror},yshift=-5pt,line width=1pt]
        ($(H1.south) + (-.2, 0)$) -- ($(HL.south) + (.2, 0)$);

    \draw[arrow] (listdstart) -- (H.north);

    \node[fill=white,draw=white] (filter) at ($(filterpt)$) {$\cdot$};

    \draw[arrow] (D2.west) -- (filter.east);
    

\end{tikzpicture}
    \caption{\parbox{0.8\textwidth}{\small \textbf{Figure 2.} Our approach to boosting modifies the function being learned by $\cW$, then feeds the resulting weak hypothesis into a decoder to produce a list $h_1, \ldots, h_L$, one of which is a strong hypothesis for the target function. Using additional samples ($S_2$), we find and output this strong hypothesis with high probability.}} 
\end{figure}
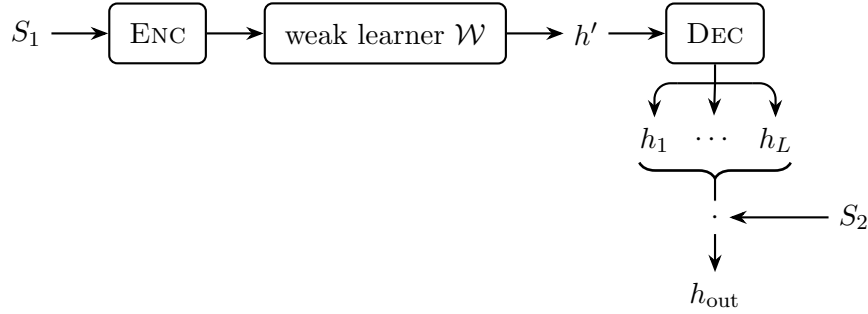

\subsection{Boosting with List-Decodable Codes}

To describe how to obtain a boosting algorithm from a list-decodable code, we first consider the simpler case when the instance space is $\pmon$. Let $\cC$ be a code represented by an encoding map $\Enc : \pmo^N \rightarrow \pmo^M$, with $N \coloneqq 2^n$. The property of these codes we are most interested in is approximate list-decodability: suppose a message sender begins with $w \in \pmo^N$ and sends $c := \Enc(w)$; if an adversary corrupts fewer than $\frac{1}{2} - \gamma$ of the bits of $c$, we want the message receiver to be able to recover a list of $L$ candidate messages $w_1, \ldots, w_L \in \pmo^N$, one of which is $\eps$-close to the original message $w$. If an algorithm for doing so exists, we say that the code is \textsl{$\eps$-approximately $(\frac{1}{2} - \gamma, L)$ list-decodable} (see, for example, \cite{IJKW}, Definition 1.5). When $M = 2^m$ for $m \in \N$, we can view strings in $\pmo^N$ and $\pmo^M$ as truth tables for functions on $n$ and $m$ bits, respectively, and thus view $\Enc$ as a map between functions. Likewise, we can view circuits $C : \pmo^n \rightarrow \pmo$ as representing messages of length $N$.

At a high level, we obtain a boosting algorithm for a class $\cF$ from a code $\cC$ as follows. Given an unknown function $f \in \cF_n$ that we want to learn, we consider a new function $f' : \pmo^m \rightarrow \pmo$ given by the encoding $f' := \Enc(f)$. Assuming that $f' \in \cF_m$ and we can generate random samples labeled by $f'$ from random examples of $f$, we can feed our $\gamma$-weak learner $\cW$ samples labeled by $f'$ to obtain a circuit $C' : \pmo^m \rightarrow \pmo$ such that $\Pr_{\by}[f'(\by) = C'(\by)] \geq \frac{1}{2} + \gamma$. Viewing the truth table of $C'$ as a corrupted version of the truth table of $f'$, this means the two have relative distance at most $\frac{1}{2} - \gamma$. Thus, if we run \Dec{} on $C'$, we obtain a list of circuits $C_1, \ldots, C_L$, among which at least one $C_i$ computes our target function $f$ with error at most $\eps$. By testing this list against additional random samples labeled by $f$, we can identify and output this high-accuracy hypothesis with high probability.

In order for the framework above to work, we require two additional properties of $\cC$ and $\cF$. First, $\cC$ should be \textsl{locally encodable}, meaning we can efficiently generate samples labeled by $\Enc(f)$ given a small number of samples labeled by $f$. Second, we need $\cF$ to be \textsl{closed under} $\cC$, meaning $\Enc(f) \in \cF$ for every $f \in \cF$. Then we have the following connection between codes and boosting.

\begin{theorem}\label{thm:list-decoder-yields-booster-intro}
    Suppose that $\cC$ is a binary code that is locally encodable and efficiently $\eps$-approximately $(\frac{1}{2} - \gamma, L)$ locally list-decodable. Then there exists an efficient boosting algorithm that, given a uniform-distribution $\gamma$-weak learner $\cW$ for a concept class $\cF$ closed under $\cC$, learns $\cF$ to accuracy $1 - O(\eps)$ over the uniform distribution with $1$ call to $\cW$ and $O(\log(L) / \eps)$ additional samples.
\end{theorem}

The relationship between the properties of the code $\cC$ and the resulting boosting algorithm can be summarized as follows.
\begin{center}
    \renewcommand{\arraystretch}{1.25}
\begin{tabular}{c c c}
    \textbf{Code $\cC$} & & \textbf{Boosting Algorithm} \\
    \hline
    $\eps$-approximately decodable & $\longleftrightarrow$ & final error $O(\eps)$ \\
    list-decodable at radius $\frac{1}{2} - \gamma$ & $\longleftrightarrow$ & $\gamma$-weak learner \\
    list size $L$ & $\longleftrightarrow$ & additional samples $O(\frac{\log L}{\eps})$ \\
    rate $N/M = 2^n/2^m$ & $\longleftrightarrow$ & input expansion $m/n$ \\
    \textsl{locally} list-decodable & $\longleftrightarrow$ & efficient
\end{tabular}
\end{center}

\paragraph{Proof of~\Cref{thm:informal-dist-specific-thm}.} To prove our distribution-specific result, we instantiate the framework above with the \Xor{} code. For a fixed $k \in \N$, the $k$-\Xor{} code encodes $f : \cX \rightarrow \pmo$ as the function $f^{\oplus k} : \cX^k \rightarrow \pmo$. Then the proof of \Cref{thm:informal-dist-specific-thm} is a straightforward combination of \Cref{thm:list-decoder-yields-booster-intro} with the list-decoding algorithm of \cite{IJKW} for the \Xor{} code.

\subsection{Standard Boosting: Distributional Codes and Arbitrary Domains} \label{sec:distributional-codes}

\paragraph{Distributional error correcting codes.} To extend this connection to the distributional setting, we generalize error correcting codes to measure corruption with respect to arbitrary distributions over the bits of the received word and message. Standard codes measure the error of a recovered message $\tilde{f}$ with respect to a message $f$ as $\Pr_{\bx \backsim \cU_n}[f(\bx) \neq \tilde{f}(\bx)]$, and measure the corruption of a received word $\tilde{g}$ from a codeword $g$ as $\Pr_{\by \backsim \cU_m}[g(\by) \neq \tilde{g}(\by)]$. To generalize to the distributional setting, a \textsl{distributional code} additionally specifies a distribution $\cD$ over $\pmo^n$ and a distribution $\cD'$ over $\pmo^m$, which allow us to measure the error between $f$ and $\tilde{f}$ as $\Pr_{\bx \backsim \cD}[f(\bx) \neq \tilde{f}(\bx)]$, and measure the corruption of $\tilde{g}$ relative to $g$ as $\Pr_{\by \backsim \cD'}[g(\by) \neq \tilde{g}(\by)]$.

Since the decoding algorithm for a distributional code must have \textsl{some} knowledge about the distribution $\cD$ in order to properly decode a corrupted message, we model this information in the simplest way possible: the decoding algorithm receives, in addition to a corrupted codeword, a small number of independent samples $\bx \backsim \cD$. The definition of $\eps$-approximate $(\frac{1}{2} - \gamma, L)$ local list-decodability for distributional codes then extends naturally from the standard setting (see \Cref{def:approx-list-decodable-distributional-code}).

\paragraph{Arbitrary domains.} We'd also like for our boosting algorithms to work over arbitrary, potentially infinite domains. Previously, we viewed an error correcting code as implicitly mapping from one function $f : \pmon \rightarrow \pmo$ to another by encoding its truth table; however, when a concept class contains functions over an infinite domain $\cX$, their truth tables are infinite strings for which standard notions of encoding and decoding no longer make sense. To capture this setting, we define codes as directly mapping a class of functions $f : \cX \rightarrow \pmo$ to $\Enc(f) : \cX'  \rightarrow \pmo$. With this change, we obtain a more general version of \Cref{thm:list-decoder-yields-booster-intro} that works over arbitrary domains and distributions, allowing us to obtain boosting algorithms in the standard setting.

\paragraph{A distributional decoder for the XOR Code.} We combine this generic connection between boosting and distributional codes with a distributional decoder for the $k$-\Xor{} code. As a mapping of functions, it encodes a function $f : \cX \rightarrow \pmo$ as $f^{\oplus k} : \cX^k \rightarrow \pmo$. We then describe and prove the correctness of an efficient distributional local list decoder for the $\Xor{}$ code. The algorithm itself is a straightforward generalization of the decoding algorithm of \cite{IJKW}, although the analysis of its correctness is fairly different. In particular, their proof crucially relies on the domain $\cX$ being discrete (see \Cref{remark:similarities-differences} for a more detailed comparison).

\paragraph{Proof of \Cref{thm:informal-main-thm}.} Combining this decoder with a distributional generalization of \Cref{thm:list-decoder-yields-booster-intro}, we obtain a distribu\-tion-free one-round boosting algorithm which has a better dependence on $\gamma$ but a worse dependence on $\eps$ than existing boosting algorithms. To get the best of both, we use our one-round booster to obtain a $.49$-weak learner from a $\gamma$-weak learner, then apply an existing boosting algorithm to achieve accuracy $1 - \eps$. This concatenated boosting algorithm is what yields \Cref{thm:informal-main-thm}.

\section{Preliminaries}

\paragraph{Notation and naming conventions.} We write $[n]$ to denote the set $\{1, 2, \ldots, n\}$. We use lowercase letters to denote bit strings, uppercase letters to denote vectors, and \textbf{boldface letters} (e.g., $\textbf{x}$, $\textbf{Y}$) to denote random variables. For an input space $\cX$ and an alphabet $\Sigma$, we write $\Sigma^{\cX}$ to denote the space of all functions $f : \cX \rightarrow \Sigma$, and we write $\Delta(\cX)$ to denote the space of all distributions on $\cX$.

Next, fix $k \in \N$ and let $\cX$ be a set. For $Z \in \cX^k$ and $I \subseteq [k]$, we let $Z_I \in \cX^{|I|}$ denote the vector 
\begin{equation*}
    Z_{I} \coloneqq (Z_{i_1}, \ldots, Z_{i_{|I|}}), \quad  I = \{i_1, \ldots, i_{|I|}\}.
\end{equation*}
If $X, Y \in \cX^{k/2}$ are two $\frac{k}{2}$-vectors and $S \subseteq [k]$ has size $|S| = \frac{k}{2}$, we denote by $\ileave_{S}(X, Y)$ the vector in $\cX^k$ that interleaves $X$ and $Y$ according to $S$, 
\begin{equation*}
    (\ileave_{S}(X, Y))_S = X, \quad (\ileave_{S}(X, Y))_{[k]\setminus S} = Y.
\end{equation*}
Finally, let $\cD$ be a distribution over $\cX$ and fix $k \in \N,$ $x \in \cX$. We denote by $\cD^{k, x}$ the distribution generated by the following process: sample a random $\bi \sim [k]$ and independent $\by_1,\ldots, \by_k \sim \cD$, then output $\bX \coloneqq (\by_1, \ldots, \by_{\bi - 1}, x, \by_{\bi + 1}, \ldots, \by_k) \in \cX^k$. When we want to remember the underlying index $\bi$ where $x$ was planted, we may write $(\bi, \bX) \sim \cD^{k, x}$.

\paragraph{Distributional error correcting codes.} Here we provide the formal definition of distributional error correcting codes as described in \Cref{sec:distributional-codes}. For two finite alphabets $\Sigma$ and $\Sigma'$ (e.g., $\pmo$), and two input spaces $\cX$ and $\cX'$, a \textsl{distributional code} is given by an encoder $\Enc : \Sigma^{\cX} \rightarrow (\Sigma')^{\cX'}$ mapping functions $f : \cX \rightarrow \Sigma$ to functions $\Enc(f) : \cX' \rightarrow \Sigma'$, as well as a message distribution $\cD \in \Delta(\cX)$ and a codeword distribution $\cD' \in \Delta(\cX')$. With this, we can generalize the definition of local list decoding to the distributional setting:

\begin{definition}[Approximately Local List-Decodable Distributional Codes] \label[definition]{def:approx-list-decodable-distributional-code}
     We say that a code $\cC = (\Enc, \cD, \cD')$ is \textup{$\eps$-approximately, $(\rho, L, u)$ locally list-decodable} if there exists a decoding algorithm $\Dec$ that, given as input $C' : \cX' \rightarrow \Sigma'$ and $u$ independent samples $x_1, \ldots, x_u \backsim \cD$, outputs a list $h_1, \ldots, h_L : \cX \rightarrow \Sigma$ with the following guarantee: for any $f : \cX \rightarrow \Sigma$ such that $\Pr_{\by \backsim \cD'}[\Enc(f)(\by) \neq C'(\by)] < \rho$, there exists some $i \in [L]$ such that $\Pr_{\bx \backsim \cD}[f(\bx) \neq h_i(\bx)] < \eps$ with probability at least $\frac{3}{4}$ over the decoder's randomness and the random draw of samples $\bx_1, \ldots, \bx_u \backsim \cD$. We say that $\Dec$ is \textup{efficient} if it runs in time polynomial in $n, \frac{1}{\gamma}, \frac{1}{\eps}$ and the size of $C'$.
\end{definition}

Additionally, we will need our codes to be locally encodable in the following sense.

\begin{definition}[local encodability] 
    We say that a code $\cC = (\Enc, \cD, \cD')$ is $\ell$-locally encodable if there exists an efficient algorithm which, for any $f : \cX \rightarrow \Sigma$, can produce a sample $\bx \backsim \cD'$ labeled by $\Enc(f)$ given $\ell$ independent random samples $\bx_1, \ldots, \bx_\ell \backsim \cD$ labeled by $f$.
\end{definition}

\paragraph{Standard learning definitions.} We will use $\cF = \{\cF_n\}_{n \in \N}$ to denote a concept class, where $\cF_n \subseteq \cF$ is the subset of $n$-arity functions in $\cF$ (e.g., those over $\pmon$ or $\R^n$). We will call $\cF_n$ the \textsl{$n$-th slice}. We will also need the following standard definitions for PAC learning and weak learning.

\begin{definition}[Distribution-specific PAC learning]
    For a concept class $\cF_n$, we say an algorithm $\cA$ learns $\cF_n$ to accuracy $1 - \eps$ over a distribution $\cD$ using $m$ samples if the following is true: for any $f \in \cF_n$, given $m$ independent samples of the form $(\bx, f(\bx))$ where $\bx \backsim \cD$, $\cA$ returns a hypothesis $h$ that, with high probability, satisfies 
    \begin{equation*}
        \Prx_{\bx \backsim \cD}[f(\bx) = h(\bx)] \geq 1 - \eps.
    \end{equation*}
    We call $\cA$ a \textup{$\gamma$-weak learner} for $\cF_n$ over $\cD$ using $m$ samples if it learns $\cF_n$ to accuracy $\frac{1}{2} + \gamma$ over $\cD$ with $m$ samples.
\end{definition}

If an algorithm $\cA$ learns $\cF_n$ to accuracy $1 - \eps$ using $m$ samples for every distribution $\cD$, we call it a \textsl{distribution-free} learning algorithm. When a learning algorithm is not defined with respect to a specific distribution, we will assume this means that it is distribution-free.

\section{Proof of \Cref{thm:list-decoder-yields-booster-intro}}

First, we state and prove in full generality our generic connection between list-decodable distributional codes and boosting.

\begin{theorem}[formal statement of \Cref{thm:list-decoder-yields-booster-intro}] \label{thm:list-decoder-yields-booster}
    Fix $\eps, \gamma > 0$, and have $\Enc : \pmo^{\cX} \rightarrow \pmo^{\cX'}$. Suppose that $\cC = (\Enc, \cD, \cD')$ is a distributional code that is $\ell$-locally encodable and $\eps$-approximat\-ely $(\frac{1}{2} - \gamma, L, u)$-list decodable, and let $\cF$ be a concept class closed under $\cC$. There exists a boosting algorithm that, given a $\gamma$-weak learner $\cW$ for $\cF$ on $\cD'$ using $s$ samples, strong learns $\cF$ to accuracy $1 - 4\eps$ over $\cD$ with one call to $\cW$, $u$ unlabeled samples from $\cD$, and $\ell \cdot s + O(\log (L) / \eps)$ labeled samples from $\cD$. 
\end{theorem}

\begin{proof}
    Suppose that $f : \cX \rightarrow \pmo$ is the unknown function we wish to learn. By the assumption that $\cF$ is closed under $\cC$, we know that $f' := \Enc(f)$ is also in $\cF$. Moreover, since $\cC$ is locally encodable, we can draw $\ell \cdot s$ labeled samples from $\cD$ to produce $s$ samples from $\cD'$ labeled by $f'$. Feeding these samples to $\cW$, we obtain a hypothesis $C' : \cX' \rightarrow \pmo$ such that $\Pr_{\by \backsim \cD'}[C'(\by) \neq f'(\by)] < \frac{1}{2} - \gamma$. Running our approximate list decoder $\Dec$ for $\cC$ with input $C'$ and unlabeled samples $\bx_1, \ldots, \bx_u \backsim \cD$, we get a list of hypotheses $h_1, \ldots, h_L : \cX \rightarrow \pmo$ such that, with probability at least $\frac{3}{4}$, there exists some $i \in [L]$ such that $\Pr_{\bx \backsim \cD}[h_i(\bx) \neq f(\bx)] < \eps$. Given this list, we will draw a test set $S$ of $t := 100 \cdot \log (L) / \eps$ labeled samples from $\cD$ and do the following: for $i = 1, 2, \ldots, L$, we test the accuracy of $h_i$ on this test set, outputting the first $h_i$ with error at most $2\eps$. To analyze the correctness of this step, we define 
    \begin{equation*}
        \Err(h_i) := \Prx_{\bx \backsim \cD}[f(\bx) \neq h_i(\bx)], \quad \Err_{S}(h_i) := \tfrac{1}{t} \cdot |\{x \in S : f(x) \neq h_i(x)\}|.
    \end{equation*}
    First, note that this algorithm only rejects a ``good" $h_i$, one with error $\leq \eps$, if $\Err_S(h_i) - \Err(h_i) > \eps$. Viewing $\Err_S(h_i)$ as the empirical mean of $t$-many i.i.d. Bernoulli random variables with mean $\Err(h_i)$, we can apply Bernstein's inequality to see that 
    \begin{equation*}
        \Prx_{\bS}\left[\Err_{\bS}(h_i) - \Err(h_i) > \eps\right] \leq \exp\left(-\frac{t\eps^2}{2\eps + \frac{2}{3}\eps}\right) = \exp\left(-\tfrac{3}{8} t\eps\right) \leq \frac{1}{8L}.
    \end{equation*}
    Thus, the probability it rejects the good hypothesis is $\leq \frac{1}{8L}$. On the other hand, it accepts a ``bad" $h_i$, one with error $\geq 4\eps$, only if $\Err_S(h_i) \leq 2\eps \leq \frac{1}{2} \cdot \Err(h_i)$. Again viewing $\Err_S(h_i)$ as the empirical mean of $t$-many i.i.d. Bernoulli random variables with mean $\Err(h_i)$, we can apply a multiplicative Chernoff bound to see that 
    \begin{equation*}
        \Prx_{\bS}\left[\Err_{\bS}(h_i) < \tfrac{1}{2} \cdot \Err(h_i)\right] \leq \exp\left(-\tfrac{1}{8}t \eps\right) \leq \frac{1}{8L}.
    \end{equation*}
    Thus, the probability that a ``bad" $h_i$ is output by our algorithm is $\leq \frac{1}{8L}$. Combining these two, conditioned on at least one hypothesis having error at most $\eps$, the probability we do not output a circuit with error at most $4\eps$ is at most $\frac{1}{4}$ by a union bound. Since the list contains a good hypothesis with probability $\geq \frac{3}{4}$, the total probability of outputting a good hypothesis is $\geq \frac{1}{2}$ (and, further, we can amplify the success probability of this algorithm to $1 - \delta$ by increasing $t$ by a multiplicative factor of $O(\log\frac{1}{\delta})$ and calling the list decoder $O(\log\frac{1}{\delta})$ times independently).  Finally, note that the only labeled samples used are the $\ell \cdot s$ required for one call to the weak learner and the $O(\log (L) / \eps)$ required for testing the list output by the decoder, and the only unlabeled samples are the $u$ used by \Dec{}.
\end{proof}

\section{Proofs of \Cref{thm:informal-main-thm,thm:informal-dist-specific-thm}}

In this section, we utilize the previous section to prove our main results. We will do so by instantiating the previous section with the $k$-\Xor{} code. For a function $f : \cX \rightarrow \pmo$, the encoder $\Enc^{\oplus k}$ of the $k$-\Xor{} code maps $f$ to $f^{\oplus k} : \cX^k \rightarrow \pmo$. For both of our boosting algorithms, we will use the following generalization of \cite{IJKW} to the distributional (and continuous domain) setting.

\begin{theorem} \label{thm:xor-decoder}
    There exists an efficient algorithm \Dec{} such that for all $\eps, \gamma > 0$ and $k \geq \Omega(\log(\frac{1}{\gamma}) / \eps)$, and any distribution $\cD$ over $\cX$, \Dec{} is an $\eps$-approximate $(\frac{1}{2} - \gamma, O(k/\gamma^2), O(k^2 \cdot \log(\frac{1}{\eps}) / \gamma^2))$ local list decoder for $(\Enc^{\oplus k}, \cD, \cD^k)$.
\end{theorem}

We will defer the proof of \Cref{thm:xor-decoder} to the next section, first using it to prove \Cref{thm:informal-main-thm,thm:informal-dist-specific-thm}. The formal statement of \Cref{thm:informal-main-thm} is the following:

\begin{theorem}[formal statement of \Cref{thm:informal-main-thm}] \label{thm:main-theorem-formal}
    There exists an efficient boosting algorithm that, for any $\gamma, \eps > 0$ and for $k := O(\log\frac{1}{\gamma})$, given a $\gamma$-weak learner $\cW$ for a concept class $\cF$ closed under $k$-\Xor{}, strong learns $\cF_n$ to accuracy $1 - \eps$ with $O(\log\frac{1}{\eps})$ calls to $\cW$ on the $(k \cdot n)$-th slice and \changedthing{$O(k^2 \cdot \log(\frac{1}{\eps}) / \gamma^2)$} additional samples.
\end{theorem}

\begin{proof}
    Suppose that $\cW$ is our $\gamma$-weak learner for $\cF$ requiring $s(\cdot)$ samples, and set $k := O(\log \frac{1}{\gamma})$ so that \Cref{thm:xor-decoder} yields a $.01$-approximate $(\frac{1}{2} - \gamma, O(k/\gamma^2), \changedthing{O(k^2 \cdot \log(\frac{1}{\eps}) / \gamma^2)})$ local list decoder for the \Xor{} code $(\Enc^{\oplus k}, \cD, \cD')$ for any distribution $\cD$. Combining this with \Cref{thm:list-decoder-yields-booster}, we obtain a distribution-free $.49$-weak learner for $\cF$ making $1$ call to $\cW$ on the $(n \cdot k)$-th slice and, simulating unlabeled samples by simply drawing labeled examples, using at most $S$ additional samples, where 
    \begin{equation*}
        S = \changedthing{O\left(\log(k / \gamma^2) + k^2/ \gamma^2\right) = O\left(k^2  / \gamma^2\right)}.
    \end{equation*}
    Call this $.49$-weak learner $\cW'$. Giving an existing boosting algorithm such as AdaBoost \cite{FREUND1997119} (viewed as a boosting-by-filtering algorithm) access to $\cW'$, we obtain a boosting algorithm which makes $O(\log\frac{1}{\eps})$ calls to $\cW'$ and learns to accuracy $1 - \eps$. Since each call to $\cW'$ requires $1$ call to $\cW$ and \changedthing{$O(k^2 / \gamma^2)$} additional samples, this composed algorithm makes $O(\log \frac{1}{\eps})$ calls to $\cW$ and uses \changedthing{$O(k^2\cdot \log(\frac{1}{\eps}) / \gamma^2)$} additional samples. 
\end{proof}

Next, we state in full generality \Cref{thm:informal-dist-specific-thm}; the proof is a straightforward instantiation of \Cref{thm:list-decoder-yields-booster} with \Cref{thm:xor-decoder}, noting that since the distribution $\cD$ is uniform over $\cX$ (and thus $\cD^k$ is uniform over $\cX^k$), it yields a uniform-to-uniform boosting algorithm which can generate samples from $\cD$ itself using random bits. 

\begin{theorem}[formal statement of \Cref{thm:informal-dist-specific-thm}] \label{thm:main-dist-specific}
    There exists a boosting algorithm that, for any $\gamma, \eps > 0$ and for $k := O(\log(\frac{1}{\gamma})/\eps)$, given a uniform-distribution $\gamma$-weak learner $\cW$ for a concept class $\cF$ closed under $k$-\Xor{}, strong learns $\cF_n$ to accuracy $1 - \eps$ over the uniform distribution with $1$ call to $\cW$ on the $(k \cdot n)$-th slice and \changedthing{$O((\log(\frac{1}{\gamma}) + \log(\frac{1}{\eps}))/\eps)$} additional samples.
\end{theorem}

\section{A List Decoder for the Distributional XOR Code}

In this section, we describe and prove the correctness of our list-decoding algorithm for the $k$-\Xor{} code. The key technical ingredient is the following list decoding algorithm for the direct product code, which, for any domain $\cX$ and alphabet $\Sigma$,\footnote{While the setting we work in is PAC learning of boolean functions, and thus all of our applications have $\Sigma = \pmo$, the proof for an arbitrary $\Sigma$ is no more complicated and could be used to extend our conceptual connection between error correcting codes and boosting to other settings, such as multi-class learning.} maps a function $f : \cX \rightarrow \Sigma$ to $f^{k} : \cX^k \rightarrow \Sigma^k$, where $f^k(x_1, \ldots, x_k) = (f(x_1), \ldots, f(x_k))$.

\begin{theorem} \label{thm:dp-decoder}
    There exists an efficient algorithm $\Dec^k$ that, for any $\eps, \gamma > 0$ and $k \geq \changedthing{512 \cdot \log(\frac{2}{\gamma}) / \eps}$, and every distribution $\cD$ over $\cX$, $\Dec^k$ is an $\eps$-approximate $(1 - \gamma, 64/\gamma, \changedthing{128 \cdot k \cdot \log(\frac{1}{\eps}) / \gamma})$ local list decoder for $(\Enc^k, \cD, \cD^k)$.
\end{theorem}

Deferring the proof of \Cref{thm:dp-decoder} for a moment, we first remark on how one obtains a distributional list-decoder for the \Xor{} code from the above. The strategy is the same as that of \cite{IJKW} and thus we will only describe it for the sake of completeness, without repeating the analysis. They prove the following:

\begin{enumerate}
    \item Given a hypothesis $C : \cX^{k} \rightarrow \pmo$ computing $f^{\oplus k}$ to accuracy $\frac{1}{2} + \gamma$ over $\cD^k$, they efficiently construct a $C' : \cX^{2k} \times \zo^{2k} \rightarrow \pmo$ that computes the function $g : \cX^{2k}  \times \zo^{2k}\rightarrow \pmo$, defined by 
    \begin{equation*}
        g(x_1, \ldots, x_{2k}, b_1, \ldots, b_{2k}) = \prod_{i = 1}^{2k} (f(x_i))^{b_i},
    \end{equation*}
    to accuracy $\frac{1}{2} + \gamma / \sqrt{k}$ over the distribution $\cD^{2k} \times U_{2k}$ (where $U_{2k}$ is the uniform distribution over $\zo^{2k}$).
    \item Given this, they use the list decoder of \cite{GL} for the \changedthing{Hadamard code} to obtain from $C'$ a hypothesis $C'' : \cX^{2k} \rightarrow \pmo^{2k}$ computing $f^{2k}$ to accuracy $\Omega(\gamma^2 / k)$ over $\cD^{2k}$. 
    \item This can then be fed into the distributional decoder in \Cref{thm:dp-decoder} to obtain a list of $L \coloneqq O(k / \gamma^2)$ hypotheses, $h_1, \ldots, h_L : \cX \rightarrow \pmo$, one of which computes $f$ to accuracy $1 - \eps$ over $\cD$ with probability at least $\frac{3}{4}$.
\end{enumerate}

Together, this entire algorithm is an efficient, $\eps$-approximate $(\frac{1}{2} - \gamma, O(k/\gamma^2), \changedthing{O(k^2 \cdot \log(\frac{1}{\eps}) / \gamma^2)})$ local list decoder for $(\Enc^{\oplus k}, \cD, \cD^k)$. Thus, establishing \Cref{thm:dp-decoder} proves \Cref{thm:xor-decoder}. We now turn our attention to proving \Cref{thm:dp-decoder}.

\subsection{A Distributional Decoder for the Direct Product Code}

Here, we describe and prove the accuracy of the algorithm realizing \Cref{thm:dp-decoder}. The decoding algorithm itself is a straightforward generalization of the decoding algorithm of \cite{IJKW} to the distributional setting. Let $f : \cX \rightarrow \Sigma$ be a function and $C : \cX^k \rightarrow \Sigma^k$ be a hypothesis computing $f^k$ to accuracy $\geq \gamma$,
\begin{equation*}
    \Prx_{\bX \sim \cD^k}\left[C(\bX) = f^k(\bX)\right] \geq \gamma.
\end{equation*}
 For $Y \in \cX^k$ and $S \subseteq [k]$ of size $|S| = \frac{k}{2}$, we will call $(Y, S)$ a \emph{pair of advice}. For any pair of advice $A$ and vector $\vec{Q} = ((Q^{(1)}, i^{(1)}), \ldots, (Q^{(t)}, i^{(t)}))$, where each $Q^{(j)} \in \cX^{k/2}$ and $i^{(j)} \in [\frac{k}{2}]$, we denote by $C_{A, \vec{Q}} : \cX \rightarrow \Sigma$ the circuit that, on input $x$, does the following:

 \begin{tcolorbox}[colback = white,arc=1mm, boxrule=0.25mm]
    \vspace{2pt} 
    \begin{enumerate}
        \item For $(Q, i) \in \vec{Q}$:
        \begin{enumerate}
            \item[1.a)] Write $Z \coloneqq (Q_1, \ldots, Q_{i - 1}, x, Q_{i + 1}, \ldots, Q_{\frac{k}{2}})$, and $X \coloneq \ileave_{S}(Y_S, Z)$, and let $i^{\star} \in [k]$ be the index that $x$ is placed at in $X$. If $C(X)_i = C(Y)_i$ for all $i \in S$, output $C(X)_{i^{\star}}$.
        \end{enumerate}
        \item Output $0$.
    \end{enumerate}
\end{tcolorbox}

\Dec{} runs by repeatedly drawing random advice then constructing the above circuit; more concretely, let $\cA$ be the distribution on advice that is sampled by drawing a random $\bY \sim \cD^k$ and uniform $\bS \subseteq [k]$ of size $|\bS| = \frac{k}{2}$, and let $\cQ$ represent the distribution $\cD^{k/2} \times [\frac{k}{2}]$. Then $\Dec$, given as input $C, \gamma, \eps,$ and $k$, runs as follows:

\begin{tcolorbox}[colback = white,arc=1mm, boxrule=0.25mm]
    \vspace{2pt} 
    \begin{enumerate}
        \item Set $t \coloneqq 64 \cdot \log(\frac{1}{\eps}) / \gamma$. Draw $3$ independent $\vec{\bQ}_1, \vec{\bQ}_2, \vec{\bQ}_3 \sim \cQ^t$.
        \item For $j \in [\frac{64}{\gamma}]$, sample a random $\bA_j \sim \cA$ and add $C_{\bA_j, \vec{\bQ}_1}, C_{\bA_j, \vec{\bQ}_{2}}, C_{\bA_j, \vec{\bQ}_{3}}$ to the list.
    \end{enumerate}
\end{tcolorbox}

Then the proof of \Cref{thm:dp-decoder} can be broken up into the following lemmas; first, we show that if a piece of advice $A$ is ``excellent'' (\changedthing{see \Cref{def:excellent-advice} below}), then an output with that advice is likely to compute $f$ to high accuracy:

\begin{lemma} \label[lemma]{lem:good-advice-good-circuit}
    Suppose that a piece of advice $A$ is excellent. Then with probability at least $\frac{1}{2}$ over a random $\vec{\bQ} \sim \cQ^t$, 
    \begin{equation*}
        \Prx_{\bx \sim \changedthing{\cD}}[C_{A, \vec{\bQ}}(\bx) = f(\bx)] \geq 1 - \eps.
    \end{equation*}
\end{lemma}

Next, we show that a randomly chosen piece of advice is likely to be excellent:

\begin{lemma} \label[lemma]{lem:lots-of-good-advice}
    A random piece of advice $\bA \sim \cA$ is excellent with probability at least $\frac{\gamma}{4}$.
\end{lemma}

Given these, the proof of \Cref{thm:dp-decoder} is simple: with probability at least $\frac{7}{8}$, \Dec{} will draw an excellent piece of advice, and with probability at least $\frac{7}{8}$, one of the $\vec{\bQ}_j$ will combine with this good advice to yield a high accuracy circuit in the outputted list. Each circuit can be constructed efficiently given $C$, $\bA_j$ and $\vec{\bQ}_i$, while each $\bA_j$ and $\vec{\bQ}_i$ can be sampled with $k$ and $k \cdot t$ unlabeled samples, respectively. Thus, \Dec{} is efficient and its overall unlabeled sample complexity is 
\begin{equation*}
    (t + 1) \cdot k \leq 128 \cdot k \cdot \frac{\log \frac{1}{\eps}}{\gamma}.
\end{equation*}
The proofs of \Cref{lem:good-advice-good-circuit,lem:lots-of-good-advice} can be found in the subsequent sections; to finish off this section, we discuss the similarities and differences between our proof strategy and that of \cite{IJKW}.

\begin{remark}[Comparison with \cite{IJKW}] \label[remark]{remark:similarities-differences}
    \textup{Our decoding algorithm is a natural generalization of \cite{IJKW} into the distributional setting, and the decision to break down the proof of \Cref{thm:dp-decoder} into \Cref{lem:good-advice-good-circuit,lem:lots-of-good-advice}, as well as the definition of ``excellent" (see \Cref{def:excellent-advice}), are very similar to analysis in \cite{IJKW}. The most significant difference between our analysis and that in \cite{IJKW} is in the proof of \Cref{lem:good-advice-good-circuit}. Both proofs use as a key claim the fact that we can approximately swap the order of conditioning when proving that good advice yields a good circuit (see \Cref{claim:excellent-cons-likely-good}); however, their proof uses properties of underlying ``sampler" graphs over collections of subsets of $\cX$ which crucially rely on $\cX$ being finite (and don't generalize to the distributional setting). On the other hand, our proof takes a more direct route,  largely making use of a concentration inequality (see \Cref{claim:B-small,claim:concentration-inequality}) to approximately exchange the order of conditioning.}
\end{remark}

\subsection{Proofs of \Cref{lem:good-advice-good-circuit,lem:lots-of-good-advice}}

First, we will introduce some notation and definitions to help in our analysis.
 For advice $A = (Y, S)$, with $Y \in \cX^k$ and $S \subseteq [k]$ with $|S| = \frac{k}{2}$, we define $N_{A} \subseteq \cX^k$ to be the set of $k$-vectors that equal $Y$ at every index in $S$, 
\begin{equation*}
    N_{A} \coloneqq \left\{X \in \cX^k : X_i = Y_i \textup{ for all $i \in S$}\right\}.
\end{equation*}
Likewise, we will have $\cN_{A}$ represent the distribution induced by $\cD^k$ on the set $N_A$. Next, we define $\Cons(A) \subseteq N_A$ to be the subset of ``consistent" vectors on which $C$ correctly computes $f$ at every index in $S$,
\begin{equation*}
    \Cons(A) \coloneqq \left\{X \in N_A : C(X)_i = f(X_i) \textup{ for all $i \in S$}\right\}.
\end{equation*}
We will denote by $\ConsFrac(A) \in \R$ the density $\ConsFrac(A) \coloneq \cN_A(\Cons(A))$. For any $X \in \cX^k$, we define $\Err(X) \subseteq [k]$ to be the subset of indices on which $C$ computes $f$ incorrectly,
\begin{equation*}
    \Err(X) = \{i \in [k] : C(X)_i \neq f(X_i)\}.
\end{equation*}
Finally, we denote by $\ErrFrac(X) \in \R$ the fractional error $\ErrFrac(X) \coloneqq |\Err(X)| / k$.

\begin{definition} \label[definition]{def:excellent-advice}
    We say that an advice pair $A = (Y, S)$ is \textup{correct} if $C(Y) = f^k(Y)$. We say that $A$ is \textup{good} if it is correct and $\ConsFrac(A) > \frac{\gamma}{2}$. Finally, we say that $A$ is \textup{excellent} if it is good and 
    \begin{equation*}
        \Ex_{\bX \sim \cN_{A}}\left[\ErrFrac(\bX) \mid \bX \in \Cons(A)\right] \leq \frac{\eps}{32}.
    \end{equation*}
\end{definition}

\subsubsection{Proof of \Cref{lem:good-advice-good-circuit}}

In this section, we prove that \changedthing{given} an excellent piece of advice $A = (Y, S)$, the output of our decoder computes $f$ to high accuracy. For the sake of notational simplicity, we will assume without loss of generality that $S = \{\frac{k}{2} + 1, \ldots, k\}$. We will also think of $\Cons(A)$ as a subset of $\frac{k}{2}$-vectors over $\cX$,
\begin{equation*}
    \Cons(A) \coloneq \{X \in \cX^{k/2} : C(X, Y_S)_i = \changedthing{C(Y)_i} \textup{ for all } i \in S\}.
\end{equation*}
Our proof of \Cref{lem:good-advice-good-circuit} will use the following claim, whose proof we defer to later in this section.

\begin{claim} \label[claim]{claim:excellent-cons-likely-good}
    Suppose that $A$ is excellent. Then
    \begin{equation}
        \Ex_{\bx \sim \cD}\left[\Prx_{(\bi, \bX) \sim \cD^{k/2, \, \bx}}\left[C(\bX, Y_S)_{\bi} \neq f(\bx) \mid \bX \in \Cons(A)\right]\right] \leq \changedthing{\frac{5\eps}{16}}. \label{eq:excellence-gives-good-circuit-target}
    \end{equation}
\end{claim}

\changedthing{In turn, the} proofs of both \Cref{claim:excellent-cons-likely-good} and \Cref{lem:good-advice-good-circuit} will use the following. Write $\mu \coloneqq \cD^{k/2}(\Cons(A))$ and define $B \subseteq \cX$ to be the subset of inputs $x$ such that $\cD^{k/2,\, x}(\Cons(A))$ is too small, 
\begin{equation*}
    B \coloneqq \left\{ x \in \cX : \cD^{k/2,\,x}(\Cons(A)) < \frac{\mu}{4}\right\}.
\end{equation*}
Then we have the following:

\begin{claim} \label[claim]{claim:B-small}
     If $A$ is good, so that $\mu \geq \frac{\gamma}{2}$, then $\cD(B) < \changedthing{\frac{\eps}{16}}$. 
\end{claim}

The proof of the above claim is a straightforward application of \Cref{claim:concentration-inequality}, whose proof can be found at the end of this section. We will first prove \Cref{lem:good-advice-good-circuit} given \Cref{claim:excellent-cons-likely-good}, then prove \Cref{claim:excellent-cons-likely-good}.\\

\begin{proof}
    We begin by bounding the following expectation:
    \begin{equation}
        \Ex_{\vec{\bQ} \sim \cQ^t}\left[\Prx_{\bx \sim \cD}\left[C_{A, \vec{\bQ}}(\bx) \neq f(\bx)\right] \right] = \Ex_{\bx \sim \cD}\left[\Prx_{\vec{\bQ} \sim \cQ^t}\left[C_{A, \vec{\bQ}}(\bx) \neq f(\bx)\right]\right]. \label{eq:good-advice-good-circuit-1}
    \end{equation}
    By switching the order of the expectation, we can now consider the behavior of the randomized hypothesis $C_{A, \vec{\bQ}}$ on a fixed input $x \in \cX$. First, we will use \Cref{claim:B-small} to focus our analysis only on those $x$ for which $\cD^{k/2,\, x}(\Cons(A))$ is relatively large, 
    \begin{align}
        \Ex_{\bx \sim \cD}\left[\Prx_{\vec{\bQ} \sim \cQ^t}\left[C_{A, \vec{\bQ}}(\bx) \neq f(\bx)\right]\right] & \leq \cD(B) + \Ex_{\bx \sim \cD}\left[\Prx_{\vec{\bQ} \sim \cQ^t}\left[C_{A, \vec{\bQ}}(\bx) \neq f(\bx)\right] \cdot \Ind[\bx \notin B]\right], \\
        & \leq \changedthing{\frac{\eps}{16}} + \Ex_{\bx \sim \cD}\left[ \Ex_{\vec{\bQ} \sim \cQ^t}\left[\Ind\left[C_{A, \vec{\bQ}}(\bx) \neq f(\bx)\right]\right] \cdot \Ind[\bx \notin B] \right]. \label{eq:good-advice-good-circuit-2}
    \end{align}
    Now, fix some $x \notin B$. We want to analyze the behavior of $C_{A, \vec{\bQ}}$ on this $x$. To do so, let $E(\vec{\bQ}, x)$ be the event that at least one of the randomly drawn $(\bQ, \bi) \in \vec{\bQ}$ has $\bZ_{(\bQ, \bi), x} \in \Cons(A)$, where we define
    \begin{equation*}
        \bZ_{(\bQ, \bi), x} \coloneq (\bQ_1, \ldots, \bQ_{\bi - 1}, x, \bQ_{\bi + 1}, \ldots, \bQ_{k/2}).
    \end{equation*}
    First, note that $C_{A, \vec{\bQ}}$ outputs some value in step 1.a) if and only if it finds some $(\bQ, \bi) \in \vec{\bQ}$ such that $\bZ_{(\bQ, \bi), x} \in \Cons(A)$. We will show that this occurs with high probability over $\vec{\bQ}$. Note that drawing a random $(\bQ, \bi) \sim \cD^{k/2} \times [\frac{k}{2}]$ and taking $\bZ_{(\bQ, \bi), x}$ is the same as drawing a random vector from $\cD^{k/2,\, x}$, so that the probability a random $\bZ_{(\bQ, \bi), x}$ is in $\Cons(A)$ is $\cD^{k/2,\, x}(\Cons(A))$. Since $x \notin B$, this is at least $\frac{\mu}{4} \geq \frac{\gamma}{8}$. Moreover, given that $\vec{\bQ}$ is drawn by taking $t$-many independent $(\bQ, \bi) \sim \cQ$, the probability that $E(\vec{\bQ}, x)$ is false over a randomly drawn $\vec{\bQ} \sim \cQ^t$ is at most
    \begin{equation*}
        \Prx_{\vec{\bQ} \sim \cQ^t}[E(\vec{\bQ}, x) \textup{ is false}] = \left(1 - \cD^{k/2, \, x}(\Cons(A))\right)^{t} \leq \left(1 - \frac{\gamma}{8}\right)^{t} \leq \changedthing{\frac{\eps}{16}},
    \end{equation*}
    since we had $t = 64 \cdot \log(\frac{1}{\eps}) / \gamma$. Thus, for any good $x \notin B$, 
    \begin{equation*}
        \Ex_{\vec{\bQ} \sim \cQ^t}\left[\Ind\left[C_{A, \vec{\bQ}}(x) \neq f(x)\right]\right] \leq \changedthing{\frac{\eps}{16}} + \Ex_{\vec{\bQ} \sim \cQ^t}\left[\Ind\left[C_{A, \vec{\bQ}}(x) \neq f(x)\right] \mid E(\vec{\bQ}, x)\right].
    \end{equation*}
    Finally, we will compute this rightmost conditional expectation. For $\vec{\bQ}$, let $(\bQ, \bi)$ be the first pair such that $\bZ_{(\bQ, \bi), x} \in \Cons(A)$. Conditioned on $E(\vec{\bQ}, x)$, such a pair exists. In this case, the probability that $C_{A, \vec{\bQ}}(x) \neq f(x)$ is the probability that $\bi \in \Err(\bZ_{(\bQ, \bi), x})$ conditioned on $\bZ_{(\bQ, \bi), x} \in \Cons(A)$; but this probability is exactly 
    \begin{equation*}
        \Prx_{(\bi, \bX) \sim \cD^{k/2, x}}\left[C(\bX, Y_S)_{\bi} \neq f(\changedthing{x}) \mid \bX \in \Cons(A)\right].
    \end{equation*}
    Having $\xi$ denote the term
    \begin{equation*}
        \xi \coloneqq \Ex_{\bx \sim \cD}\left[ \Ex_{\vec{\bQ} \sim \cQ^t}\left[\Ind\left[C_{A, \vec{\bQ}}(\bx) \neq f(\bx)\right]\right] \cdot \Ind[\bx \notin B] \right], 
    \end{equation*}
    we have that
    \begin{align*}
        \xi & \leq \changedthing{\frac{\eps}{16}} + \Ex_{\bx \sim \cD}\left[\Prx_{(\bi, \bX) \sim \cD^{k/2, x}}\left[C(\bX, Y_S)_{\bi} \neq f(\bx) \mid \bX \in \Cons(A)\right] \cdot \Ind[\bx \notin B]\right], \\
        & \leq \changedthing{\frac{\eps}{16}} + \Ex_{\bx \sim \cD}\left[\Prx_{(\bi, \bX) \sim \cD^{k/2, x}}\left[C(\bX, Y_S)_{\bi} \neq f(\bx) \mid \bX \in \Cons(A)\right]\right] \leq \changedthing{\frac{\eps}{16} + \frac{5\eps}{16}},
    \end{align*}
    where the last inequality comes from applying \Cref{claim:excellent-cons-likely-good}. Substituting this into \Cref{eq:good-advice-good-circuit-1,eq:good-advice-good-circuit-2}, we have that 
    \begin{equation*}
        \Ex_{\vec{\bQ} \sim \cQ^t}\left[\Prx_{\bx \sim \cD}\left[C_{A, \vec{\bQ}}(\bx) \neq f(\bx)\right] \right] = \Ex_{\bx \sim \cD}\left[\Prx_{\vec{\bQ} \sim \cQ^t}\left[C_{A, \vec{\bQ}}(\bx) \neq f(\bx)\right]\right] \leq \changedthing{\frac{\eps}{16} + \frac{\eps}{16} + \frac{5\eps}{16}} \leq \frac{\eps}{2}.
    \end{equation*}
    Applying Markov's inequality, we see that with probability at least $\frac{1}{2}$ over $\vec{\bQ} \sim \cQ^t$, 
    \begin{equation*}
        \Prx_{\bx \sim \cD}\left[C_{A, \vec{\bQ}}(\bx) \neq f(\bx)\right] \leq \eps,
    \end{equation*}
    which completes the proof.
\end{proof}

Next, we prove \Cref{claim:excellent-cons-likely-good}. \\

\begin{proof}
    For any $x \in \cX$, we will define \changedthing{$\Phi, \hat{\Phi} : \cX \rightarrow \R$} as follows,
    \begin{align*}
        \Phi(x) & \coloneqq \Prx_{(\bi, \bX) \sim \cD^{k/2, x}}\left[\changedthing{C(\bX, Y_S)_{\bi} \neq f(x)} \mid \bX \in \Cons(A)\right], \\
        \hat{\Phi}(x) & \coloneqq \Prx_{(\bi, \bX) \sim \cD^{k/2, x}}\left[\changedthing{C(\bX, Y_S)_{\bi} \neq f(x)} \textup{ and } \bX \in \Cons(A)\right].
    \end{align*}
    Then we can rewrite the lefthand side of \Cref{eq:excellence-gives-good-circuit-target} using \Cref{claim:B-small},
    \begin{align}
        \Ex_{\bx \sim \cD}\left[\Phi(\bx)\right] & = \Ex_{\bx \sim \cD}\left[\Phi(\bx) \cdot \Ind[\bx \in B]\right] + \Ex_{\bx \sim \cD}\left[\Phi(\bx) \cdot \Ind[\bx \notin B]\right], \\
        & \leq \changedthing{\frac{\eps}{16}} + \Ex_{\bx \sim \cD}[\Phi(\bx) \cdot \Ind[\bx \notin B]], \\
        & = \changedthing{\frac{\eps}{16}} + \Ex_{\bx \sim \cD}\left[\frac{\hat{\Phi}(\bx)}{\cD^{k/2, \bx}(\Cons(A))} \cdot \Ind[\bx \notin B]\right], \\
        & \leq \changedthing{\frac{\eps}{16}} + \frac{4}{\mu} \cdot \Ex_{\bx \sim \cD}[\hat{\Phi}(\bx)]. \label{eq:excellence-gives-good-circuit-flip-1}
    \end{align}
    Now, we can expand and rearrange this expectation as follows,
    \begin{align*}
        \Ex_{\bx \sim \cD}\left[\hat{\Phi}(x)\right] & = \int_{\cX} \hat{\Phi}(x) \, d\cD(x), \\
        & = \int_{\cX}\left(\frac{2}{k} \cdot \sum_{i = 1}^{\frac{k}{2}} \int_{\cX^{k/2 - 1}}\Ind\left[i \in \Err(Z_{(X, i), x}) \wedge Z_{(X, i), x} \in \Cons(A)\right] \, d \cD^{k/2 - 1}(X)\right) \, d \cD(x), \\
        & = \frac{2}{k} \cdot \sum_{i = 1}^{\frac{k}{2}} \int_{\cX} \int_{\cX^{k / 2 - 1}} \Ind\left[i \in \Err(Z_{(X, i), x}) \wedge Z_{(X, i), x} \in \Cons(A)\right] \, d\cD^{k / 2 - 1}(X) \, d \cD(x), \\
        & = \frac{2}{k} \cdot \sum_{i = 1}^{\frac{k}{2}} \int_{\cX^{k/2}} \Ind\left[i \in \Err(X) \wedge X \in \Cons(A)\right] \, d\cD^{k / 2}(X), \\
        & = \int_{\cX^{k/2}} \left(\frac{2}{k} \cdot \sum_{i = 1}^{\frac{k}{2}} \Ind\left[i \in \Err(X) \wedge X \in \Cons(A)\right] \right) \, d\cD^{k / 2}(X), \\
        & = \int_{\cX^{k/2}} \left(\Ind[X \in \Cons(A)] \cdot \changedthing{2 \cdot \ErrFrac(X)} \right) \, d\cD^{k/2}(X), \\
        & = \changedthing{2 \cdot} \Ex_{\bX \sim \cD^{k/2}}\left[\Ind[\bX \in \Cons(A)] \cdot \ErrFrac(\bX)\right] = \changedthing{2 \cdot \mu} \cdot \Ex_{\bX \sim \cD^{k/2}}\left[\ErrFrac(\bX) \mid \bX \in \Cons(A)\right].
    \end{align*}
    Substituting this into \Cref{eq:excellence-gives-good-circuit-flip-1} and applying the fact that $A$ is excellent, we see that
    \begin{equation*}
        \Ex_{\bx \sim \cD}\left[\Phi(\bx)\right] \leq \changedthing{\frac{\eps}{16} + 8 \cdot \Ex_{\bX \sim \cD^{k/2}}\left[\ErrFrac(\bX) \mid \bX \in \Cons(A)\right] \leq \frac{5\eps}{16}},
    \end{equation*}
    completing the proof.
\end{proof}

Finally, we prove \Cref{claim:B-small}. We do so by noting that it is a nearly direct consequence of the following, more general concentration inequality, which we prove below.

\begin{claim} \label[claim]{claim:concentration-inequality}
    Suppose that $\Omega$ is a set, $d \in \N$, and let $\mu$ be a distribution over $\Omega$. Let $G \subseteq \Omega^d$ have density $\beta \coloneqq \mu^d(G)$. For any $x \in \Omega$ and $i \in [d]$, we define $\rho_i : \Omega \rightarrow [0,1]$ by
    \begin{equation*}
        \rho_i(x) \coloneqq \Prx_{\bX \sim \mu^{(i, x)}}\left[\bX \in G\right], \quad \mu^{(i, x)} \coloneqq \mu  \times \cdots \mu \times \underbrace{\delta_x}_{\textup{$i$-th pos}} \times \, \mu \cdots \times \mu,
    \end{equation*}
    where $\delta_x$ is the point mass distribution at $x$. Additionally, define $\rho : \Omega \rightarrow [0,1]$ to be the average $\rho(x) \coloneqq \frac{1}{d} \cdot \sum_{i = 1}^d \rho_i(x)$. Then
    \begin{equation*}
        \mu\left(\left\{x \in \Omega : \rho(x) < \frac{\beta}{4}\right\}\right) < 16 \cdot \frac{\log \frac{1}{\beta}}{d}.
    \end{equation*}
\end{claim}

\begin{proof}
    First, we can rewrite the lefthand side of our target inequality as 
    \begin{align*}
        \mu\left(\left\{x \in \Omega : \rho(x) < \frac{\beta}{4}\right\}\right) & = \Ex_{\bx \sim \mu}\left[\Ind\left[\frac{1}{d} \cdot \sum_{i = 1}^d \rho_i(\bx) < \frac{\beta}{4}\right]\right].
    \end{align*}
    Note that for any $x \in \Omega$, if $\rho(x) < \frac{\beta}{4}$, then $\rho_i(x) < \frac{\beta}{2}$ for at least half of $i \in [d]$; thus, we have that 
    \begin{equation*}
        \Ind\left[\frac{1}{d} \cdot \sum_{i = 1}^d \rho_i(x) < \frac{\beta}{4}\right] \leq \frac{2}{d} \cdot \sum_{i = 1}^d \Ind\left[\rho_i(x) < \frac{\beta}{2}\right].
    \end{equation*}
    We can then combine the two equations above as follows,
    \begin{align}
        \mu\left(\left\{x \in \Omega : \rho(x) < \frac{\beta}{4}\right\}\right) & \leq \Ex_{\bx \sim \mu} \left[\frac{2}{d} \cdot \sum_{i = 1}^d \Ind\left[\rho_i(\bx) < \frac{\beta}{2}\right]\right], \\
        & = \frac{2}{d} \cdot \sum_{i = 1}^{d} \Ex_{\bx \sim \mu} \left[\Ind\left[\rho_i(\bx) < \frac{\beta}{2}\right]\right], \\
        & = \frac{2}{d} \cdot \sum_{i = 1}^d \mu\left(\left\{x \in \Omega : \rho_i(x) < \frac{\beta}{2}\right\}\right). \label{eq:concentr-ineq-targ-1}
    \end{align}
    For the sake of brevity, for each $i \in [d]$, we will define $B_i \subseteq \Omega$ by
    \begin{equation*}
        B_i \coloneqq \left\{x \in \Omega : \rho_i(x) < \frac{\beta}{2}\right\},
    \end{equation*}
    so that the righthand side of \Cref{eq:concentr-ineq-targ-1} is equal to $\frac{2}{d} \cdot \sum_{i = 1}^d \mu(B_i)$. Our goal is now to bound this sum; to do so, we will use the following auxillary distributions in our analysis. Let $\cG$ be the distribution induced by $\mu^d$ on $G$. Viewing this as a distribution over $\Omega^d$, we can also consider its marginal distributions $\cG_1, \ldots, \cG_d$ over $\Omega$. We proceed by bounding the KL-divergence of $\cG$ and $\mu^d$ on both sides; on one hand, we see that 
    \begin{equation*}
        \DKL(\cG \, \| \, \mu^d) = \Ex_{\bX \sim \cG}\left[\log\left(\frac{d\cG}{d\mu^d}(\bX)\right)\right] = \Ex_{\bX \sim \cG}\left[\log \left(\tfrac{1}{\beta}\right)\right] = \log \tfrac{1}{\beta}.
    \end{equation*}
    On the other hand, by the sub-additivity of KL-divergence from a product distribution, we have 
    \begin{equation*}
        \DKL(\cG \, \| \, \mu^d) \geq \sum_{i = 1}^d \DKL(\cG_i \,\|\, \mu). 
    \end{equation*}
    Combining the two and dividing by a factor of $d$, we have that 
    \begin{equation}
        \frac{1}{d} \sum_{i = 1}^d \DKL(\cG_i \,\|\, \mu) \leq \frac{\log \frac{1}{\beta}}{d}. \label{eq:concentr-ineq-total-dkl-bound}
    \end{equation}
    Next, we will bound $\mu(B_i)$ by $\DKL(\cG_i \,\|\, \mu)$. To do so, we can expand $\DKL(\cG_i \,\|\, \mu)$ as follows,
    \begin{align*}
        \DKL(\cG_i \,\|\, \mu) & = \Ex_{\bx \sim \cG_i}\left[\log\left(\frac{d\cG_i}{d\mu}(\bx)\right)\right], \\
        & = \Ex_{\bx \sim \cG_i}\left[\log\left(\frac{d\cG_i}{d\mu}(\bx)\right) \cdot \Ind[\bx \in B_i]\right] + \Ex_{\bx \sim \cG_i}\left[\log\left(\frac{d\cG_i}{d\mu}(\bx)\right) \cdot \Ind[\bx \notin B_i]\right], \\
        & \geq \cG_i(B_i) \cdot \log\left(\frac{\cG_i(B_i)}{\mu(B_i)}\right) + (1 - \cG_i(B_i)) \cdot \log\left(\frac{1 - \cG_i(B_i)}{1 - \mu(B_i)}\right),
    \end{align*}
    where the last inequality comes from applying the log sum inequality to each expectation. We will denote $\delta_i \coloneqq \mu(B_i)$ and define $g_{\delta_i} : [0,1] \rightarrow \R$ by 
    \begin{equation*}
        g(x) \coloneq x \cdot \log\left(\frac{x}{\delta_i}\right) + (1 - x) \cdot \log\left(\frac{1 - x}{1 - \delta_i}\right).
    \end{equation*}
    With this, we can rewrite the preceding inequality as
    \begin{equation}
        \DKL(\cG_i \,\|\, \mu) \geq g_{\delta_i}(\cG_i(B_i)). \label{eq:concentr-ineq-dkl-i-bound}
    \end{equation}
    To further bound this, we will use the following bound on $\cG_i(B_i)$ in terms of $\mu(B_i)$. To do so, we first note that, using Bayes' theorem,\footnote{Although the equation is stated using conditional probabilities which may be $0$ for continuous distributions, we note that the same end-to-end equality holds by replacing those conditional probabilities with conditional densities.}
    \begin{equation*}
        \frac{d\cG_i}{d\mu}(x) = \frac{\Pr_{\bX \sim \mu^d}[\bX_i = x \mid \bX \in \cG]}{\mu(x)} = \frac{\Pr_{\bX \sim \changedthing{\mu^{(i, x)}}}[\bX \in G] \cdot \mu(x)}{\mu(x) \cdot \mu^d(G)} = \frac{\rho_i(x)}{\beta}. 
    \end{equation*}
    Thus, we see that 
    \begin{align*}
        \cG_i(B_i) & = \int_{\Omega}\Ind[x \in B_i] \, d\cG_i(x) = \int_{\Omega} \Ind[x \in B_i] \cdot \frac{\rho_i(x)}{\beta} \, d\mu(x) \leq \frac{1}{2} \cdot \int_{\Omega} \Ind[x \in B_i] \, d\mu(x) = \frac{\mu(B_i)}{2}.
    \end{align*}
    Thus, $\cG_i(B_i) \in [0, \frac{\delta_i}{2}]$, and so we can bound \Cref{eq:concentr-ineq-dkl-i-bound} below by
    \begin{equation*}
        \DKL(\cG_i \, \| \, \mu) \geq g_{\delta_i}(\cG_i(B_i)) \geq \min_{x \in [0, \frac{\delta_i}{2}]} g_{\delta_i}(x).
    \end{equation*}
    Now, note that the derivative of $g_{\delta_i}(x)$ with respect to $x$ is 
    \begin{equation*}
        \frac{\partial}{\partial x} g_{\delta_i}(x) = \log\left(\frac{x}{1 - x} \cdot \frac{1-\delta_i}{\delta_i}\right),
    \end{equation*}
    which, since $\frac{x}{1 - x} \leq \frac{\delta_i}{1 - \delta_i}$ for $x \leq \frac{\delta_i}{2}$, is always negative on $[0, \frac{\delta_i}{2}]$. Thus $g_{\delta_i}$ is minimized at the boundary of this interval, and so we know that
    \begin{align*}
        \DKL(\cG_i \,\|\, \mu) \geq \min_{x \in [0, \frac{\delta_i}{2}]} g_{\delta_i}(x) & \geq g_{\delta_i}\left(\frac{\delta_i}{2}\right), \\
        & = \frac{\delta_i}{2} \cdot \log\left(\frac{1}{2}\right) + \left(1 - \frac{\delta_i}{2}\right) \cdot \log\left(\frac{1 - \frac{\delta_i}{2}}{1 - \delta_i}\right), \\
        & = \frac{\delta_i}{2} \cdot \log\left(\frac{1}{2}\right) + \left(1 - \frac{\delta_i}{2}\right) \cdot \log\left(1 + \frac{\delta_i}{2 \cdot (1 - \delta_i)}\right), \\
        & \geq \frac{\delta_i}{2} \cdot \log\left(\frac{1}{2}\right) + \left(1 - \frac{\delta_i}{2}\right) \cdot \left(\frac{\delta_i}{2\cdot(1 - \delta_i)} \cdot \frac{1}{1 + \frac{\delta_i}{2(1-\delta_i)}} \cdot \frac{1}{\ln(2)}\right), \\
        & = \frac{\delta_i}{2} \cdot \log\left(\frac{1}{2}\right) + \left(1 - \frac{\delta_i}{2}\right) \cdot \left(\frac{\delta_i}{2} \cdot \frac{1}{1 - \frac{\delta_i}{2}} \cdot \frac{1}{\ln(2)}\right), \\
        & = \frac{\delta_i}{2} \cdot \left(\frac{1}{\ln(2)} - 1\right) \geq \frac{\delta_i}{8}.
    \end{align*}
    Combining this with \Cref{eq:concentr-ineq-targ-1,eq:concentr-ineq-total-dkl-bound} and recalling that $\delta_i = \mu(B_i)$, we have that 
    \begin{equation*}
        \mu\left(\left\{x \in \Omega : \rho(x) < \frac{\beta}{4}\right\}\right) \leq \frac{2}{d} \cdot \sum_{i = 1}^d \delta_i \leq \frac{16}{d} \cdot \sum_{i = 1}^d \DKL(\cG_i \,\|\, \mu) \leq 16 \cdot \frac{\log \frac{1}{\beta}}{d},
    \end{equation*}
    which completes the proof.
\end{proof}

\subsubsection{Proof of \Cref{lem:lots-of-good-advice}}

In this section, we prove that a random piece of advice is excellent with probability at least $\frac{\gamma}{4}$.\\

\begin{proof}
    First, we will bound the probability that a random piece of advice $\bA \sim \cA$ is good. Since $C$ computes $f^k$ to accuracy $\geq \gamma$, we know that a random $\bA \sim \cA$ is correct with probability at least $\gamma$. On the other hand, the probability that a random $\bA \sim \cA$ is correct but not good is
    \begin{align*}
        \Prx_{\bA \sim \cA}[\bA \textup{ is correct but not good}] & = \Prx_{\bA \sim \cA}\left[\bY \textup{ is correct } \wedge \,\, \ConsFrac(\bA) < \tfrac{\gamma}{2}\right].
    \end{align*}
    While we had $\cA$ denote the distribution on advice which first samples $\bY \sim \cD^k$ then $\bS \sim \binom{k}{k/2}$ independently, it's equivalent to first sample $\bS \sim \binom{k}{k/2}$, then draw $\bY_1, \bY_2 \sim \cD^{k/2}$ independently, and finally set $\bY \coloneqq \ileave_{\bS}(\bY_1, \bY_2)$. With this view, we have
    \begin{align*}
        \Prx_{\bA \sim \cA}\left[\bY \textup{ correct } \wedge \,\, \cN_{\bA}(\Cons(\bA)) < \tfrac{\gamma}{2}\right] & = \Prx_{\bS, \bY_1, \bY_2}\left[\bY \textup{ correct } \wedge \ConsFrac(\bY_1, \bS) < \tfrac{\gamma}{2}\right] \\
        & \leq \Prx_{\bS, \bY_1, \bY_2}\left[\bY \textup{ correct} \mid \ConsFrac(\bY_1, \bS) < \tfrac{\gamma}{2}\right].
    \end{align*}
    Conditioned on $\bS$ and $\bY_1$, $\bY \coloneq \ileave_{\bS}(\bY_1, \bY_2)$ is distributed according to $\cN_{(\bY_1, \bS)}$. Moreover, $\bY$ is correct only if $\bY \in \Cons(\bY_1, \bS)$. Conditioned on $\ConsFrac(\bY_1, \bS) < \frac{\gamma}{2}$, this occurs with probability $\leq \frac{\gamma}{2}$. Thus, the righthand side of the above equation is at most $\frac{\gamma}{2}$, so that 
    \begin{equation}
        \Prx_{\bA \sim \cA}[\bA \textup{ is good}] \geq \Prx_{\bA \sim \cA}[\bA \textup{ is correct}] - \Prx_{\bA \sim \cA}[\bA \textup{ is correct but not good}] \geq \gamma - \tfrac{\gamma}{2} \geq \tfrac{\gamma}{2}. \label{eq:many-excellent-good-bound}
    \end{equation}
    To bound the probability that a random $\bA \sim \cA$ is excellent, we will bound the probability that $\bA$ is good but not excellent. Let $E_0(A)$ be the event that $A$ is good, and have $E_1(A)$ represent the event that $A$ is good but not excellent. We will write $\alpha \coloneqq \frac{\eps}{32}$. Note that if $A$ is good but not excellent, then 
    \begin{equation*}
        \Ex_{\bX \sim \cN_A}\left[\ErrFrac(\bX) \mid \bX \in \Cons(\bA)\right] \geq \alpha,
    \end{equation*}
    which, in particular, implies that 
    \begin{equation*}
        \Ex_{\bX \sim \cN_A}\left[\ErrFrac(\bX) \cdot \Ind\left[\ErrFrac(\bX) \geq \tfrac{\alpha}{2}\right]\mid \bX \in \Cons(\bA)\right] \geq \tfrac{\alpha}{2}.
    \end{equation*}
    We can use this to bound the probability that an \changedthing{advice pair} is good but not excellent as follows,
    \begin{align*}
        \Prx_{\bA \sim \cA}[E_1(\bA)] & = \Ex_{\bA \sim \cA}\left[\Ind\left[E_0(\bA) \text{ but} \Ex_{\bX \sim \cN_{\bA}}\left[\ErrFrac(\bX) \mid \bX \in \Cons(\bA)\right] \geq \alpha \right]\right], \\
        & \leq \Ex_{\bA \sim \cA}\left[\Ind\left[E_0(\bA) \text{ but} \Ex_{\bX \sim \cN_{\bA}}\left[\ErrFrac(\bX) \cdot \Ind\left[\ErrFrac(\bX) \geq \tfrac{\alpha}{2}\right]\mid \bX \in \Cons(\bA)\right] \geq \tfrac{\alpha}{2}\right]\right],\\
        & \leq \frac{2}{\alpha} \cdot \Ex_{\bA \sim \cA}\left[\Ind[E_0(\bA)] \cdot \Ex_{\bX \sim \cN_{\bA}}\left[\ErrFrac(\bX)\cdot\Ind\left[\ErrFrac(\bX) \geq \tfrac{\alpha}{2}\right]\mid \Cons(\bA)\right]\right], \\
        & = \frac{2}{\alpha} \cdot \Ex_{\bA \sim \cA}\left[\frac{\Ind[E_0(\bA)]}{\ConsFrac(\bA)} \cdot \Ex_{\bX \sim \cN_{\bA}}\left[\ErrFrac(\bX) \cdot \Ind\left[\bX \in \Cons(\bA) \wedge\ErrFrac(\bX) \geq \tfrac{\alpha}{2}\right]\right]\right], \\
        & \leq \frac{4}{\alpha \cdot \gamma} \cdot \Ex_{\bA \sim \cA}\left[\Ex_{\bX \sim \cN_{\bA}}\left[\Ind\left[E_0(\bA) \wedge \bX \in \Cons(\bA) \wedge \ErrFrac(\bX) \geq \tfrac{\alpha}{2}\right] \cdot \ErrFrac(\bX)\right]\right], \\
        & \leq \frac{4}{\alpha \cdot \gamma} \cdot \Prx_{\bA \sim \cA,\, \bX \sim \cN_{\bA}, \bi \sim [k]}\left[E_0(\bA) \wedge \bX \in \Cons(\bA) \wedge \ErrFrac(\bX) \geq \tfrac{\alpha}{2} \wedge \bi \in \Err(\bX)\right], \\
        & \leq \frac{4}{\alpha \cdot \gamma} \cdot \underbrace{\Prx_{\bA \sim \cA,\, \bX \sim \cN_{\bA}, \bi \sim [k]}\left[\bA \text{ is correct}\wedge \bX \in \Cons(\bA) \wedge \ErrFrac(\bX) \geq \tfrac{\alpha}{2} \wedge \bi \in \Err(\bX)\right]}_{\xi}.
    \end{align*}
    We will now reinterpret this last probability as follows. Rather than drawing $\bA \sim \cA$ then $\bX \sim \cN_{\bA}$, we could instead do the following: draw $\bY \sim \cD^k$, a random subset $\bS \sim \binom{k}{k/2}$, and a random $\bZ \sim \cD^{k/2}$, then set $\bX = \bY$, and $\bA = (\ileave_{\bS}(\bY_{\bS}, \bZ), \bS)$. In this case, if $\bA$ is correct and $\bX \in \Cons(\bA)$, then $\bS$ must be disjoint from $\Err(\bX)$; thus, we can bound this probability as 
    \begin{align*}
        \xi & \leq \Prx_{\bY \sim \cD^k, \, \bS \sim \binom{k}{k/2}, \, \bZ \sim \cD^{k/2}, \bi \sim [k]}\left[\bS \cap \Err(\bY) = \emptyset \, \wedge \, \ErrFrac(\bY) \geq \tfrac{\alpha}{2} \wedge \bi \in \Err(\bY)\right], \\
        & = \Ex_{\bY \sim \cD^k}\left[\Prx_{\bS \sim \binom{k}{k/2}, \bZ \sim \cD^{k / 2}}[\changedthing{\bS} \cap \Err(\bY) = \emptyset \wedge \ErrFrac(\bY) \geq \tfrac{\alpha}{2}] \cdot \ErrFrac(\bY)\right], \\
        & = \Ex_{\bY \sim \cD^k}\left[\Prx_{\bS \sim \binom{k}{k/2}, \bZ \sim \cD^{k / 2}}[\changedthing{\bS} \cap \Err(\bY) = \emptyset \mid \ErrFrac(\bY) \geq \tfrac{\alpha}{2}]  \cdot \Ind[\ErrFrac(\bY) \geq \tfrac{\alpha}{2}] \cdot \ErrFrac(\bY)\right], \\
        & \leq \Ex_{\bY \sim \cD^k}\left[\changedthing{\exp\left(-\ErrFrac(\bY) \cdot k / 2\right)}  \cdot \Ind[\ErrFrac(\bY) \geq \tfrac{\alpha}{2}] \cdot \ErrFrac(\bY)\right], \\
        &  \leq \Ex_{\bY \sim \cD^k}\left[\frac{\alpha \cdot \gamma^2}{16}\right] \leq \frac{\alpha \cdot \gamma^2}{16}.
    \end{align*}
    Note that the second to last inequality came from the fact that $\ErrFrac(Y) \geq \frac{\alpha}{2}$ and, for any $\rho \in (\alpha / 2, 1]$, \changedthing{$\rho \cdot e^{-k\rho/2} \leq \alpha \cdot \gamma^2 / 16$}. Thus, we have that 
    \begin{equation*}
        \Prx_{\bA \sim \cA}[\bA \textup{ is good but not excellent}] \leq \frac{4}{\alpha \cdot \gamma} \cdot  \frac{\alpha \gamma^2}{16} \leq \frac{\gamma}{4}.
    \end{equation*}
    Combining this with \Cref{eq:many-excellent-good-bound}, we can complete the proof, 
    \begin{equation*}
        \Prx_{\bA \sim \cA}[\bA \textup{ is excellent}] \geq \Prx_{\bA \sim \cA}[\bA \textup{ is good}] - \Prx_{\bA \sim \cA}[\bA \textup{ is good but not excellent}] \geq \tfrac{\gamma}{2} - \tfrac{\gamma}{4} = \tfrac{\gamma}{4},
    \end{equation*}
    which completes the proof.
\end{proof}

\section*{Acknowledgments}

We thank the COLT reviewers for helpful feedback and suggestions. The authors are supported by NSF awards 1942123, 2211237, 2224246, a Sloan Research Fellowship, and Omer Reingold's Simons Foundation investigators award.

\bibliographystyle{alpha}
\bibliography{ref}

\end{document}